\documentclass[letterpaper, 10 pt, conference]{ieeeconf}  % Comment this line out

\IEEEoverridecommandlockouts                              % This command is only needed if 
\usepackage[english]{babel}
\usepackage[utf8]{inputenc}
\usepackage[T1]{fontenc}
\usepackage{mathtools}
\usepackage[running]{lineno}
\usepackage{lmodern}
\usepackage{latexsym,amsmath,amssymb,amsfonts,graphicx}
\usepackage{epsfig}
\usepackage{setspace}
\usepackage{bm}
\usepackage{hyperref}
\usepackage[autolinebreaks,useliterate]{mcode}

\newcommand*\patchAmsMathEnvironmentForLineno[1]{%
      \expandafter\let\csname old#1\expandafter\endcsname\csname #1\endcsname
      \expandafter\let\csname oldend#1\expandafter\endcsname\csname end#1\endcsname
      \renewenvironment{#1}%
         {\linenomath\csname old#1\endcsname}%
         {\csname oldend#1\endcsname\endlinenomath}}%
    \newcommand*\patchBothAmsMathEnvironmentsForLineno[1]{%
      \patchAmsMathEnvironmentForLineno{#1}%
      \patchAmsMathEnvironmentForLineno{#1*}}%
    \AtBeginDocument{%
    \patchBothAmsMathEnvironmentsForLineno{equation}%
    \patchBothAmsMathEnvironmentsForLineno{align}%
    \patchBothAmsMathEnvironmentsForLineno{flalign}%
    \patchBothAmsMathEnvironmentsForLineno{alignat}%
    \patchBothAmsMathEnvironmentsForLineno{gather}%
    \patchBothAmsMathEnvironmentsForLineno{multline}%
    }
\usepackage[textwidth=7in,textheight=9in,centering]{geometry}
\newcommand{\at}{\makeatletter @\makeatother}

\newtheorem{examplex}{Example}[section]
  {\popQED\endexamplex}
\usepackage{amssymb, amsmath, graphicx, enumerate}
\usepackage{biblatex}
\usepackage{csquotes}
\DeclareMathOperator*{\argmin}{\arg\!\min}
\usepackage{algorithm}
\usepackage[noend]{algpseudocode}
\newcommand{\bi}{\begin{itemize}}
\newcommand{\ei}{\end{itemize}}
\newcommand{\vo}[1]{\boldsymbol{#1}}
\newcommand{\x}{\vo{x}}

\newcommand{\y}{\vo{y}}

\newcommand{\trace}[1]{\mathbf{tr}\left(#1\right)}

\newtheorem{theorem}{Theorem}

\usepackage{biblatex,xcolor}
\definecolor{darkgreen}{RGB}{0,128,0}

\NewDocumentCommand \vTh {} {\boldsymbol{\theta}}
\NewDocumentCommand \vMu {} {\boldsymbol{\mu}}
\NewDocumentCommand \vMuP {} {\boldsymbol{\mu}_p}
\NewDocumentCommand \vMuQ {} {\boldsymbol{\mu}_q}
\NewDocumentCommand \vMuKF {} {\boldsymbol{\mu}^{\text{KF}}}
\NewDocumentCommand \vXBar {} {\bar{\boldsymbol{x}}}

\title{\LARGE \bf
Variational Kalman Filtering with H$_{\infty}$-Based Correction for Robust Bayesian Learning in High Dimensions}

\author{Niladri Das, Jed A. Duersch, and Thomas A. Catanach % <-this % stops a space
\thanks{ Niladri Das (corresponding author, ndas\at sandia.gov),  Jed Duersch (jaduers\at sandia.gov), and Thomas A. Catanach (tacatan\at sandia.gov) are with Sandia National Laboratories, Livermore, CA 94550, USA.  Sandia National Laboratories is a multimission laboratory managed and operated by National Technology \& Engineering Solutions of Sandia, LLC, a wholly owned subsidiary of Honeywell International Inc., for the U.S. Department of Energy’s National Nuclear Security Administration under contract DE-NA0003525. This paper describes objective technical results and analysis. Any subjective views or opinions that might be expressed in the paper do not necessarily represent the views of the U.S. Department of Energy or the United States Government.}}

\begin{document}

\include{def}
\def\dispmuskip{\thinmuskip= 3mu plus 0mu minus 2mu \medmuskip=  4mu plus 2mu minus 2mu \thickmuskip=5mu plus 5mu minus 2mu}
\def\textmuskip{\thinmuskip= 0mu                    \medmuskip=  1mu plus 1mu minus 1mu \thickmuskip=2mu plus 3mu minus 1mu}
\def\beq{\dispmuskip\begin{equation}}    \def\eeq{\end{equation}\textmuskip}
\def\beqn{\dispmuskip\begin{displaymath}}\def\eeqn{\end{displaymath}\textmuskip}
\def\bea{\dispmuskip\begin{eqnarray}}    \def\eea{\end{eqnarray}\textmuskip}
\def\bean{\dispmuskip\begin{eqnarray*}}  \def\eean{\end{eqnarray*}\textmuskip}

%def\paradot#1{\vspace{1.3ex plus 0.7ex minus 0.5ex}\noindent{\bf\boldmath{#1.}}}

\maketitle
\pagestyle{empty}

%%%%%%%%%%%%%%%%%%%%%%%%%%%%%%%%%%%%%%%%%%%%%%%%%%%%%%%%%%%%%%%%%%%%%%%%%%%%%%%%
\begin{abstract}
In this paper, we address the problem of convergence of sequential variational inference filter (VIF) through the application of a robust variational objective and H$_{\infty}$-norm based correction for a linear Gaussian system. As the dimension of state or parameter space grows, performing the full Kalman update with the dense covariance matrix for a large scale system requires increased storage and computational complexity, making it impractical. The VIF approach, based on mean-field Gaussian variational inference, reduces this burden through the variational approximation to the covariance usually in the form of a diagonal covariance approximation. The challenge is to retain convergence and correct for biases introduced by the sequential VIF steps. We desire a framework that improves feasibility while still maintaining reasonable proximity to the optimal Kalman filter as data is assimilated. To accomplish this goal, a H$_{\infty}$-norm based optimization perturbs the VIF covariance matrix to improve robustness. This yields a novel VIF-H$_{\infty}$ recursion that employs consecutive variational inference and H$_{\infty}$ based optimization steps. We explore the development of this method and investigate a numerical example to illustrate the effectiveness of the proposed filter.
\end{abstract}
%%%%%%%%%%%%%%%%%%%%%%%%%%%%%%%%%%%%%%%%%%%%%%%%%%%%%%%%%%%%%%%%%%%%%%%%%%%%%%%%
\section{INTRODUCTION}
The sequential estimation of a system's state is essential for many problems in science and engineering. Filtering allow sequential observations to be integrated into a time-varying estimate of a system's state. A large number of such systems can be modelled as a linear Gaussian system. The Kalman filter (KF), derived from the optimal Bayesian filter, is used for state/parameter estimation of such problems. As the dimension of the system becomes very large, the standard formulations of KF become computationally intractable due to matrix storage and computation requirements. We anticipate this task becoming increasingly problematic  as traditional data assimilation methods, like Kalman filters, are used to sequentially update very high-dimensional machine learning models, as is the case in learning for deep neural networks \cite{kovachki2019ensemble}. Therefore, the main challenge is to develop an approximating algorithm that is robust enough to asymptotically converge to the data generating process. 

There are variants of KF that are computationally efficient for high dimensional problems. The reduced rank Kalman Filter \cite{Cane_1996,Dee_1991,Voutilainen_2007} searches for lower dimensional subspace to operate in. However, their performance assumes there exists a low enough dimensional subspace capturing the bulk of the system's structure while still computationally tractable. This often does not hold in very high dimensions. Further, these methods suffer from the lack of structured mechanism to find the effective lower dimensional projection operator. Added to that, such operators are typically fixed in time, unable to switch to a different optimal lower dimensional space. Direct approximation techniques \cite{Auvinen_2009,auvinen2009large,Bardsley_2011} have properties that are favourable in high dimensions. However, the previously proposed methods provide no guarantee that the approximated covariance matrix remains non-negative definite.

In this paper we propose a variational formulation aided by the H$_{\infty}$ filter to avoid the complexity of storage and computation with large dense state error covariance matrices. Variational methods, which optimize an approximating distribution to a Bayesian posterior, have become critical in large scale machine learning because they can handle both high dimensional models and large amounts of data \cite{zhang2018advances, blei2017variational}. Variational methods often approximate the posterior distribution with a multivariate normal distribution with diagonal covariance, making it tractable to store and manipulate in high dimensions \cite{thomas2018gaussian}. We explore combining the robustness of H$_{\infty}$ and the tractability of variational inference to provide an algorithm that is efficient and can handle uncertainty due to approximations.

The outline is as follows: First in Section II we introduce the model along with traditional KF, then in Sections III and IV we discuss the VI and H$_{\infty}$ filter respectively.  In Section V we derive our proposed augmented filter. Next, in the Section VI we demonstrate the effectiveness of our technique and finally in Section VII, we provide a discussion of the limitations and possible future directions.

\section{The Finite-Dimensional Linear-Gaussian Model and Kalman Filter}
Many large scale systems, where the problem entails state or parameter estimation, can be modeled or approximated by a linear system with additive Gaussian noise. 
We enumerate observations and update by the index $t$, so that $\vo{x}_t \in \mathbb{R}^n$ represents input variables that predicts outputs $\vo{y}_t\in \mathbb{R}^m$. We are particularly interested in the case where the dimension of inputs is much larger than dimension of outputs, $m\ll n$. We are going to focus on the simplest case where $m = 1$, so our measurement and model update equations becomes,
% Let $(\x_i,\y_i)$ denote the sequential input-output pair to a linear Gaussian system $\mathcal{S}$ from the product space of $\mathbb{R}^n\times \mathbb{R}^m$, where $m << n$.
%\begin{subequations}
\begin{linenomath}
\begin{align}
    \y_t &= \x_t^T\vTh_t + \vo{\eta}_t \label{meas}\\
    \vTh_{t+1} &= \vo{A}_{t}\vTh_{t} + \vo{w}_{t} \label{dyn}
\end{align}
\label{eq:linear_gaussian}
\end{linenomath}
where $\vTh_{t}\in \mathbb{R}^n$ is the model parameter and/or states, that is to be estimated. The variables $\vo{w}_{t}$ and $\vo{\eta}_t$ are assumed to be i.i.d. zero mean Gaussian noise with covariance matrix $\vo{Q}$ and $\vo{R}$ respectively. The matrix $\vo{A}_{t}$ is identity and $\vo{w}_t = 0$ for all $i$ when we perform parameter estimation. We assume a prior on $\vTh\sim \mathcal{N}(\vo{0},\vo{\Sigma}_{\vTh_0})$.

% The Kalman update for the mean estimate :
% \begin{subequations}
% \begin{align}
%     \vMu_{t+1} &=\vMu_{t} + \vo{K}(\y-\x_t^T\vMu_{t})\label{KF_mean}\\
%     \vo{K} &= \sP_t\x_t(\x_t^T\sP_t\x_t+r)^{-1}\label{KF_K}
% \end{align}
% \end{subequations}
% The VI update for the mean estimate :
% \begin{subequations}
% \begin{align}
%     \tilde{\vMu}_{t+1} &=\tilde{\vMu}_{t} + \tilde{\vo{K}}(\y-\x_t^T\tilde{\vMu}_{t})\label{VI_mean}\\
%     \tilde{\vo{K}} &= \tilde{\sP}_t\x_t(\x_t^T\tilde{\sP}_t\x+r)^{-1}\label{VI_K}
% \end{align}
% \end{subequations}

Under the model \eqref{eq:linear_gaussian}, the optimal update scheme based on Bayesian inference is to use the Kalman Filter. Given a current estimate $\vTh_{t-1} \sim \mathcal{N} \left (\vMu_{t-1}, \vo{P}_{t-1} \right )$, the KF defines how to update our estimate given $\vo{A}_{t}$, $\vo{x}_t$, and $\vo{y}_i$ to arrive at $\vTh_{i} \sim \mathcal{N} \left (\vMuKF_t, \vo{P}^{\text{KF}}_t \right )$. If we have a set of efficient covariance matrix approximations, e.g. diagonal matrices $\vo{\mathcal{P}}$, then the approximate KF problem relies on finding a $\vo{P}_t \in \vo{\mathcal{P}}$ that captures information about $\vo{P}^{\text{KF}}_t$ without introducing bias (difference between expected estimate and the true value) in the filter. A motivating approach is to choose $\vo{P}_i$ such that $\vo{P}^{\text{KF}}_t \preccurlyeq \vo{P}_t \preccurlyeq \vo{P}_{t-1}$. $\vo{P}_t \preccurlyeq \vo{P}_{t-1}$ implies information has been learned while $\vo{P}^{\text{KF}}_t \preccurlyeq \vo{P}_t$ implies that the inference is conservative. However, this approach is too conservative to enable efficient learning in the case of diagonal matrices. Finding a satisfactory $\vo{P}_t$ is near impossible because the principle axes of $\vo{P}^{\text{KF}}_t$ almost never align with the standard basis so $\vo{P}_{t-1} = \vo{P}_t$ (Theorem \ref{noSDP}). This motivates the need for a less conservative yet robust approach.

\begin{theorem}
For parameter estimation (e.g. $A_t = \mathbb{I}$ and $w_t = 0$) with a random input vector $\vo{x}_t \sim \mathcal{N} \left (\bar{x},\Sigma_x\right)$, the constraint $\vo{P}^{\text{KF}}_t \preccurlyeq \vo{P}_t \preccurlyeq \vo{P}_{t-1} \xRightarrow{a.s.} \vo{P}_t = \vo{P}_{t-1}$, when $\vo{P}_t$, $\vo{P}_{t-1}$ are diagonal.
\label{noSDP}
\end{theorem}
\begin{proof}
Let $\vo{d}^{t-1}$, $\vo{d}^{t}$ be the diagonals of $\vo{P}_{t-1}$ and $\vo{P}_{t}$. Then $\vo{P}_t \preccurlyeq \vo{P}_{t-1} \implies d^{t}_j \leq d^{t-1}_j \forall j$. For the KF, $\vo{P}^{\text{KF}}_t = \vo{P}_{t-1} - \frac{\vo{P}_{t-1} \vo{x}^T_{t} \vo{x}_{t}  \vo{P}_{t-1}}{\vo{x}_{t}\vo{P}_{t-1} \vo{x}^T_{t} +R_{t}}$. $\vo{P}^{\text{KF}}_t \preccurlyeq \vo{P}_t \Leftrightarrow \vo{u} \left (\vo{P}_{t-1} - \vo{P}_{t} \right ) \vo{u}^T \leq \vo{u} \frac{\vo{P}_{t-1} x^T_{t} x_{t}  \vo{P}_{t-1}}{x_{t}\vo{P}_{t-1} x^T_{t} +R_{t}} \vo{u}^T$ $\forall \vo{u} \in \mathbb{R}^n$. Since $\vo{P}_{t-1}$ and $\vo{P}_{t}$ are diagonal, the LHS is $\sum_j u_j^2 (d^{t-1}_j - d^{t}_j)$. Because $\vo{P}_{t-1} \vo{x}^T_{t} \vo{x}_{t}  \vo{P}_{t-1}$ is a rank-1 matrix, $\exists \vo{u}_{\perp} \in Null \left (\vo{P}_{t-1} \vo{x}^T_{t} \vo{x}_{t}  \vo{P}_{t-1} \right)$ s.t. the RHS is 0. When this holds then, $\sum_j u_j^2 ({d}^{t-1}_j - {d}^{t}_j) \leq 0$. Since both $u_j^2$ and $({d}^{t-1}_j - {d}^{t}_j)$ are always non-negative, then either $u_j^2$ or $({d}^{t-1}_j - {d}^{t}_j)$ must be $0$ $\forall j$. Since $\vo{x}_t$ are randomly distributed then the resulting $u_{\perp}$ will almost surely be all non-zero. Therefore a.s.,$({d}^{t-1}_j = {d}^{t}_j) \implies \vo{P}_{t-1} = \vo{P}_{t}$.
\end{proof}

\section{Discrete-time VI Filter}
Variational Inference(VI) is an approximate inference method in Bayesian statistics. Given a model we infer its posterior density by updating our prior belief about the model with observations. Solving for the exact posterior $p(\vTh \mid \vo{y})$ is often impossible, necessitating the use of different types of approximations such as approximating it with samples or a more tractable distribution. In VI this approximation is done with a distribution $q(\vTh \mid \vo{\phi})$ which comes from a set of possible distributions parameterized by $\vo{\phi}$.

For linear Gaussian models, the analytic posterior maybe constructed; however, the cost of storing $\vo{P}_t$ and potentially computing its inverse $\vo{P}_t^{-1}$ for subsequent analysis has motivated us to simplify the structure of the Gaussian distribution. We assume a variational structure such that the individual $\vTh_t$ elements are independent. This method is called the mean-field approximation \cite{blei2017variational}, which imposes a diagonal structure on $\vo{P}_t=\text{diag}(\vo{d}_t)$. This is much more tractable in high dimensional problems. In this case the variational distribution is parameterized by the mean and diagonal vector $\vo{d}_t$.

Traditional VI seeks to minimize the Kullback–Leibler (KL) divergence between the approximating variational distribution and the target posterior. For two distribution $p \left ( \vTh \right )$ and $q \left ( \vTh \right )$ the KL divergence is given by,

\begin{equation}
   \text{D}_{\text{KL}}(p(\vTh)||q(\vTh)) = \int p(\vTh) \log \frac{p \left ( \vTh \right )}{q \left ( \vTh \right )} d\vTh.
\end{equation}
The KL divergence stems from information theory and is similar to a distance between the distributions; it is nonnegative and only zero if the distributions are the same almost everywhere. Unlike a distance metric, however, it is not symmetric nor does it obey the triangle inequality. If both $p$ and $q$ are n-dimensional multivariate Gaussian distributions, with means and covariances $(\vMuP,\vo{\Sigma}_p)$ and $(\vMuQ,\vo{\Sigma}_q)$ respectively, then the KL-divergence becomes
\begin{align}
   \text{D}_{\text{KL}}(p(\vTh)||q(\vTh)) &= \nonumber \\\frac{1}{2} \biggl ( \trace{\vo{\Sigma}_q^{-1} \vo{\Sigma}_p} &- n
   +||\vMuP -\vMuQ||^2_{\vo{\Sigma}_q^{-1}} + \log \frac{|\vo{\Sigma}_q|}{|\vo{\Sigma}_p|} \biggr ).
\label{eq:guass_kl}
\end{align}

When using the KL divergence for VI there is a choice to minimize either direction of the KL divergence e.g. $\text{D}_{\text{KL}}(p(\vTh \mid \vo{y})||q(\vTh \mid \vo{\phi}))$ or $\text{D}_{\text{KL}}(q(\vTh \mid \vo{\phi})||p(\vTh \mid \vo{y}))$. Choosing $\text{D}_{\text{KL}}(q(\vTh \mid \vo{\phi})||p(\vTh \mid \vo{y}))$ is common because it does not require sampling or otherwise manipulating the posterior $p(\vTh \mid \vo{y})$. This objective is equivalent to maximizing the evidence lower bound (ELBO) so is often referred to as ELBO optimization. For Gaussian VI with mean-field approximation we optimize this KL divergence over the positive space of $\vo{d}_t$.
\begin{align}
    \mathcal{L}(\vo{d}_t,\vo{y}_t,\vo{x}_t) = \text{D}_{\text{KL}}(q(\vTh \mid \vo{d}_t)||p(\vTh_t \mid \vo{y}_t))
\end{align}
where $q(\vTh \mid \vo{d}_t)$ is the Gaussian VI distribution parameterized by $\vo{d}_t$. The $p(\vTh_t \mid \vo{y}_t)$ is the posterior distribution of the state given the observation. We assume that the $p,q$ have the same mean so it is not a parameter of the optimization because it can be easily shown that the optimal mean is the posterior mean. Therefore, we would use the standard Kalman update to update the mean as data is assimilated. 

Similarly, we could also choose the other direction of the KL divergence, $\text{D}_{\text{KL}}(p(\vTh \mid \vo{y})||q(\vTh \mid \vo{\phi}))$. Within VI this is know as Expectation Propagation (EP). While EP is more challenging, it may be preferred when tractable because it is well known that the ELBO formulation of VI generally underestimates uncertainty \cite{blei2017variational} compared to EP. Therefore sequential ELBO-VI is more likely to lead to bias compared to EP-VI. Fortunately, when the dimension of the streaming data is small, EP can be computed analytically because the posterior is just a low-rank update to the mean and covariance. Therefore, for the remainder of this work we will focus on the EP formulation. 

\begin{theorem}[Optimal $P_{\text{VI-EP}}$]
Suppose the target posterior $p\left(\vTh \mid \vo{y} \right) \sim \mathcal{N} \left (\vMuKF, \vo{P}_{\text{KF}} \right)$ and the variational distribution is $q(\vTh \mid \vo{\phi}) \sim \mathcal{N} \left(\vMu, \vo{P}(\vo{d}) \right)$, where $\vo{P}(\vo{d})$ is constrained to be a diagonal matrix with diagonal elements $d_i$. Then the EP optimal $\vo{\mu}_{\text{VI-EP}}$ and $\vo{d}_{\text{VI-EP}}$ are given by
\begin{align}
\left( \vo{\mu}_{\text{VI-EP}}, \vo{d}_{\text{VI-EP}} \right) &= \argmin_{(\vo{\mu}, \vo{d})} \text{D}_{\text{KL}}(p(\vTh \mid \vo{y})||q(\vTh \mid \vo{\mu}, \vo{d}))\nonumber\\
&= \left( \vMuKF, \text{diag}(\vo{P}_{\text{KF}}) \right)
\end{align}
\end{theorem}
\begin{proof}
The optimal $\vo{\mu}$ is found by setting the $\vo{\mu}$'s to be equal and eliminating the norm term in \eqref{eq:guass_kl}. The optimal $\vo{d}$ is found by setting the derivative of $\eqref{eq:guass_kl}$ wrt $d_i$ to zero and solving for $d_i$.
\end{proof}

\begin{figure}
   \centering
   \includegraphics[width=0.5\textwidth]{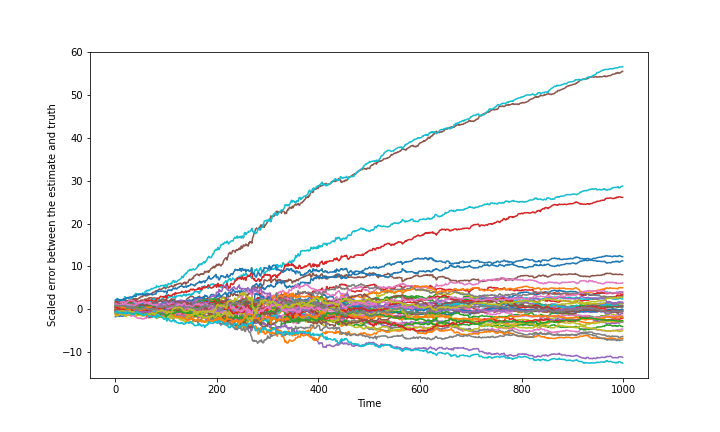}
   \caption{Plot showing the scaled estimation error response of the VI-EP filter with dimension of $\vTh$, $n=50$ and dimension of $\vo{y}$, $m=1$. The scaling is wrt to the standard deviation from the covariance matrix. We see some state are not converging because their error relative to the standard deviation is growing and approaches 60$\sigma$ by the end of the simulation. This indicates that not only are they being poorly estimated but that their uncertainty is very poorly quantified.}
   \label{VIEP}
\end{figure}

Even EP-VI filter suffers from convergence issue (having bias) when the $\vTh$ dimension is considerably higher than the observation dimension (Fig. \ref{VIEP}), unlike in \cite{thomas2018gaussian} where both dimensions are comparable. This is because learning sequentially from a limited number of observations requires capturing parameter correlation which is cannot be done with a diagonal matrix. Therefore, the covariance update must be more conservative than the EP update.

With the same diagonal structure of $\vo{P}_t$ we considered the general family of f-divergence to measure the difference between $p$ and $q$. 
In \cite{e22010108}, the notions of information density $\mathbb{D}(q|p)$ and $\mathbb{L}^{r}$ information psuedometrics,
\begin{align}
    \mathbb{L}^r = \Big(\int p(\vTh) \Big|\log \frac{p \left ( \vTh \right )}{q \left ( \vTh \right )}\Big|^r d\vTh\Big)^{1/r}; \quad r\geq 1
\end{align}
enables us to investigate the performance of this information theory based f-divergence measure to identify the VI distribution. In fig.~\ref{inforpseudo} we see that optimizing the information pseudometric for $r=2$, results in an approximation that captures more of the high posterior probability domain, which dominates the objective. As a consequence, it also includes regions with lower posterior probability than traditional VI EP or the ELBO optimization.  
\begin{figure}
   \centering
   \includegraphics[width=0.45\textwidth]{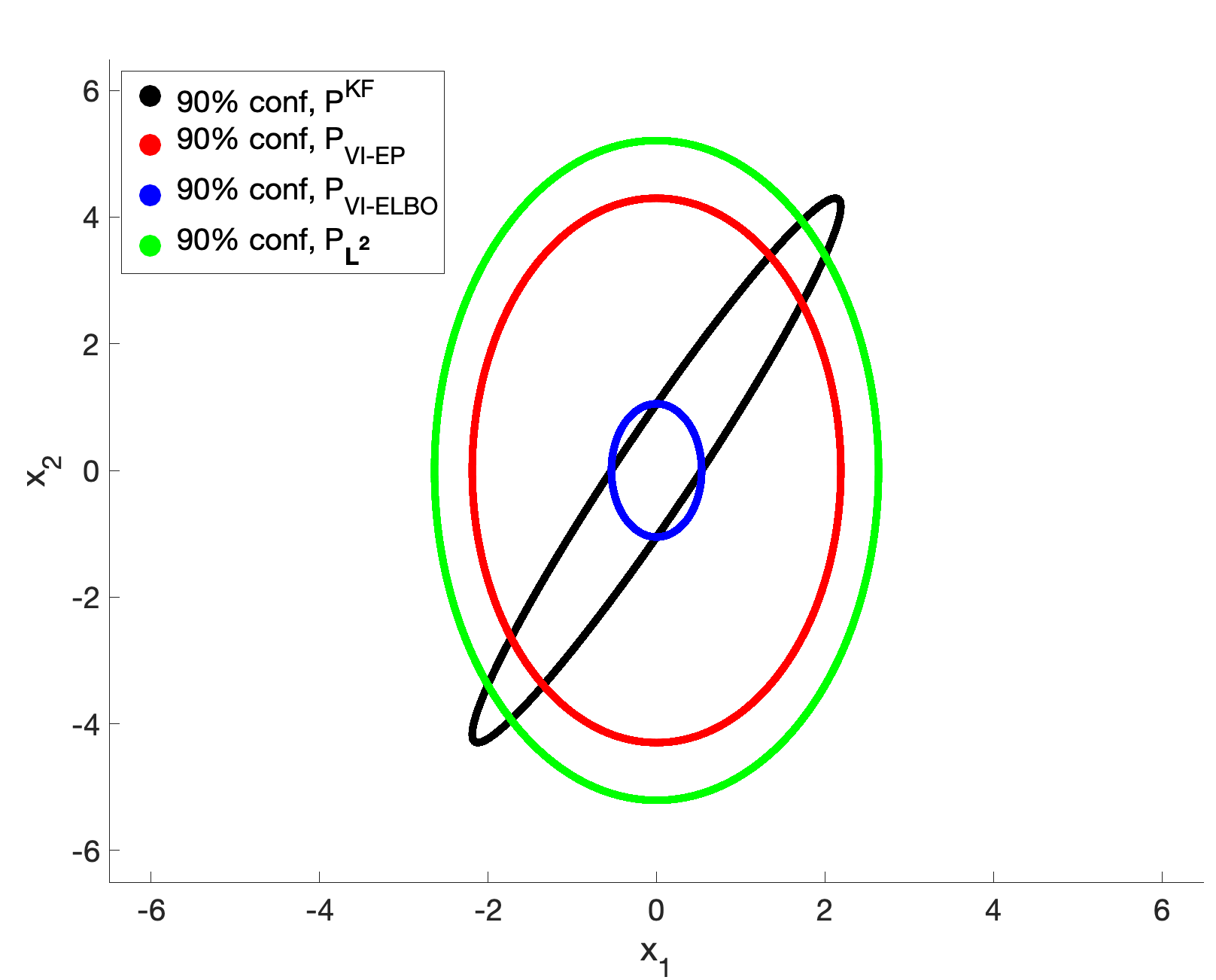}
   \caption{Comparing 90\% confidence level of true posterior (in black), traditional VI update based on KL divergence with EP (in red) and with the ELBO (in blue). VI with the information pseudometric update using $r=2$ (in green). The pseudometric optimum avoids suppressing probability in domains that have significant posterior density. }
   \label{inforpseudo}
\end{figure}
Motivated by this observation, we propose a filter, whose covariance update step is given by the following,

$\mathbb{L}^{r}$ Filter Update Step:
The diagonal posterior covariance $\vo{P}_t(\vo{d})$, is the outcome of the following information pseudometric optimization,
\begin{align}
\argmin_{\vo{d}_i} \int p(\vTh) \Big|\log \frac{p \left ( \vTh \right )}{q \left ( \vTh \mid \vo{d} \right )}\Big|^r d\vTh = \vo{d}^* \rightarrow \vo{P}_t(\vo{d}^*) \label{eq:therorem2}
\end{align}
%\end{theorem}

For high-dimensional models, the diagonal structure of the $\vo{P}_t$ can severely constrain the movement of estimates towards the true values. Due to the lack of cross-correlation, information gain only occurs along standard basis dimensions, causing some estimates to become much more confident than others. To mitigate this issue, and enable each of the elements in the estimate of $\vTh_t$ to be confident, we draw motivation from the H$_{\infty}$ filter, where we increase robustness is exchanged for slower information gain. For the remainder of this work we will focus on the $\mathbb{L}^2$ formulation. For if $p,q$ are both n-dimensional multivariate Gaussian distributions with common mean $\vMu$ and but different covariances $\vo{\Sigma}_p$ and $\vo{\Sigma}_q=\text{diag}(\vo{d})$ respectively, then equation.~\eqref{eq:therorem2} for $r=2$ simplifies to,

\begin{subequations}
\begin{align}
\int p(\vTh) \Big|\log &\frac{p \left ( \vTh \right )}{q \left ( \vTh \mid \vo{d} \right )}\Big|^2 d\vTh = \frac{n}{2} - \trace{\vo{M}}  + \frac{\trace{\vo{M}^2}}{2} \nonumber \\
&+ \frac{1}{4} \left (\trace{\vo{M}} - \log |\vo{M}| - n \right )^2
% &\vo{M} = \vo{\Sigma}_q^{-1}(\vo{d})\vo{\Sigma}_p
\end{align}\label{eq:L2expo}%
\end{subequations}
where $\vo{M} = \vo{\Sigma}_q^{-1}(\vo{d})\vo{\Sigma}_p$.
With the $\vo{d}$ exposed in equations.~\eqref{eq:L2expo}, we optimize the $\mathbb{L}^2$ psuedometric with respect to $\vo{d}$ to calculate  $\vo{d}^*$. This objective and its gradients can be efficiently computed when $\vo{\Sigma}_p$ is diagonal plus low rank.

% \comment[nil]{Insert the expression from the code}

\section{Discrete-time $H_{\infty}$ filter}
We consider the same set of equations \eqref{dyn} and \eqref{meas}, and add that $\vTh_{t}$ itself is the signal of interest rather than its linear projection. The goal of the H$_{\infty}$ filter\cite{simon2006optimal} is to estimate $\vTh_{t}$, such that the cost function,
\begin{align}
    J = \frac{\sum \limits_{t=0}^{N-1}||\vTh_{t}-\hat{\vTh}_{t}||^2_{\vo{S}_t}}{||\vTh_{0}-\hat{\vTh}_{0}||^2_{\vo{P}_0^{-1}}+ \sum \limits_{t=0}^{N-1}\Big(||\vo{w}_{t}||^2_{\vo{Q}^{-1}}+||\vo{\eta}_{t}||^2_{\vo{R}^{-1}}\Big)}\label{hinfcost}
\end{align}
can be made less than $1/\gamma$, where $\vo{S}_t$ is a symmetric, positive definite matrix, chosen to be identity. The posterior covariance matrix of the state/parameter estimate is $\vo{P}_t$.  The estimator that achieves this is,
\begin{subequations}
\begin{align}
    \vo{P}_t &= \tilde{\vo{P}}_t[\vo{I}-\gamma \vo{S}_t\tilde{\vo{P}}_t + \vo{x}_t\vo{R}^{-1}\vo{x}_t^T\tilde{\vo{P}}_t]^{-1}\\
    \vo{K}_t &= \vo{P}_t\vo{x}_t\vo{R}^{-1}\label{eq:Khinf}\\
    \hat{\vTh}_{t} &= \tilde{\vTh}_{t} + \vo{K}_t(\vo{y}_t-\x_t^T\tilde{\vTh}_t)
    % \vo{P}_{t+1} &= \vo{A}_{t}\vo{P}_{t}[\vo{I}-\gamma \vo{S}_t\vo{P}_t + \vo{x}_t\vo{R}^{-1}\vo{x}_t^T\vo{P}_t]^{-1}\vo{A}_{t}^T + \vo{Q}
\end{align}
\end{subequations}
where $\tilde{\vo{P}}_t$ is the prior covariance and $\hat{\vTh}_{t}$ is the posterior estimate from the prior $\tilde{\vTh}_{t}$. The condition, $\tilde{\vo{P}}_{t}^{-1} - \gamma \vo{I} + \vo{x}_t\vo{R}^{-1}\vo{x}_t^T \succ 0$ must hold for the estimator to provide solution to the $H_{\infty}$ problem.
% \begin{align}
%     \tilde{\vo{P}}_{t}^{-1} - \gamma \vo{I} + \vo{x}_t\vo{R}^{-1}\vo{x}_t^T \succ 0.
% \end{align}
% We consider the choice of $\vo{S}_i=\vo{I}$ for simplicity; however, in principle any diagonal $\vo{S}_i$ would be tractable in our framework. Since $\vo{x}_i\vo{R}^{-1}\vo{x}_i^T$ is known at each step, we absorb this term in the $\gamma$ and have a much simpler representation,
% \begin{align}
%     \vo{K}_i &= \vo{P}_i[\vo{I}-\gamma\vo{P}_i]^{-1}\vo{x}_i\vo{R}^{-1}\label{HinfK}
% \end{align}

\section{Augmented H$_{\infty}$ Filter Update Step}
The $\gamma$ parameter in the H$_{\infty}$ Filter determines the upper bound on the H$_{\infty}$ cost function in \eqref{hinfcost}. We seek to calculate $\gamma$ to minimize the discrepancy between the Kalman gains resulting from the regular Kalman update and the H$_{\infty}$ Kalman gain calculated after the VI or the $\mathbb{L}^2$ stage. Doing so we are ensuring the diagonal structure over the $\vo{P}_t$ while improving the robustness of the estimates. The proposed $\mathbb{L}^2$-VI H$_{\infty}$ filter algorithm is presented in algorithm.~\ref{algo:b}. The algorithm drastically simplifies for static parameter estimation for which $\vo{A}_t$ is identity and $\vo{w}_t=0$.
% \begin{theorem}[$\mathbb{L}^r$ H$_{\infty}$Filter Update Step]\ The posterior diagonal covariance matrix $\vo{P}^{H_{\infty}}_i$, is the outcome of the two optimization steps, the information pseudometric based one, followed by the H$_{\infty}$ estimator based, starting with the regular Kalman update from diagonal $\vo{P}^{H_{\infty}}_{i-1}$,
% \begin{align}
%   \vo{K}_i &= (\vo{A}_{i-1}\vo{P}^{H_{\infty}}_{i-1}\vo{A}_{i-1}^T+\vo{Q})\x_i\nonumber\\&\times\Big[\x_i^T(\vo{A}_{i-1}\vo{P}^{H_{\infty}}_{i-1}\vo{A}_{i-1}^T+\vo{Q})\x_i+\vo{R}\Big]^{-1}\nonumber\\
%   \vo{P}_i^{\text{KF}} &= (\vo{I}-\vo{K}_i\vo{x}_i^T)(\vo{A}_{i-1}\vo{P}^{H_{\infty}}_{i-1}\vo{A}_{i-1}^T+\vo{Q})
% \end{align}

% \noindent Optimization 1:
% \begin{align}
%  \argmin_{\vo{d}_i} \int dq(\vo{d}_i)\Big|\text{log}\frac{q(\vo{d}_i)}{p(\vTh_i,\vo{y}_i)}\Big|^r = \vo{d}_i^{\mathbb{L}^r} \rightarrow \vo{P}^{\mathbb{L}^r}_i(\vo{d}_i)
% \end{align}

% \noindent Compute the Kalman update for the next step using $\vo{P}_i^{\text{KF}}$:
% \begin{align}
% \vo{K}^{\text{KF}}_{i+1} &= (\vo{A}_{i}\vo{P}^{\text{KF}}_{i}\vo{A}_{i}^T+\vo{Q})\x_{i+1}\nonumber\\&\times\Big[\x_{i+1}^T(\vo{A}_{i}\vo{P}^{\text{KF}}_{i}\vo{A}_{i}^T+\vo{Q})\x_{i+1}+\vo{R}\Big]^{-1}\nonumber
% \end{align}

% \noindent Optimization 2 (from \eqref{HinfK}):
% \begin{align}
%   &\gamma^*_i = \argmin_{\gamma} ||\vo{P}^{\mathbb{L}^r}_i[\vo{I}-\gamma\vo{P}^{\mathbb{L}^r}_i]^{-1}\vo{x}_{i+1}\vo{R}^{-1}-\vo{K}^{\text{KF}}_{i+1}||_2 \nonumber\\
%   &\vo{P}^{H_{\infty}}_i  = \vo{P}^{\mathbb{L}^r}_i[\vo{I}-\gamma^*_i\vo{P}^{\mathbb{L}^r}_i]^{-1}
% \end{align}
% \end{theorem}

\begin{figure}
   \centering
   \includegraphics[width=0.5\textwidth]{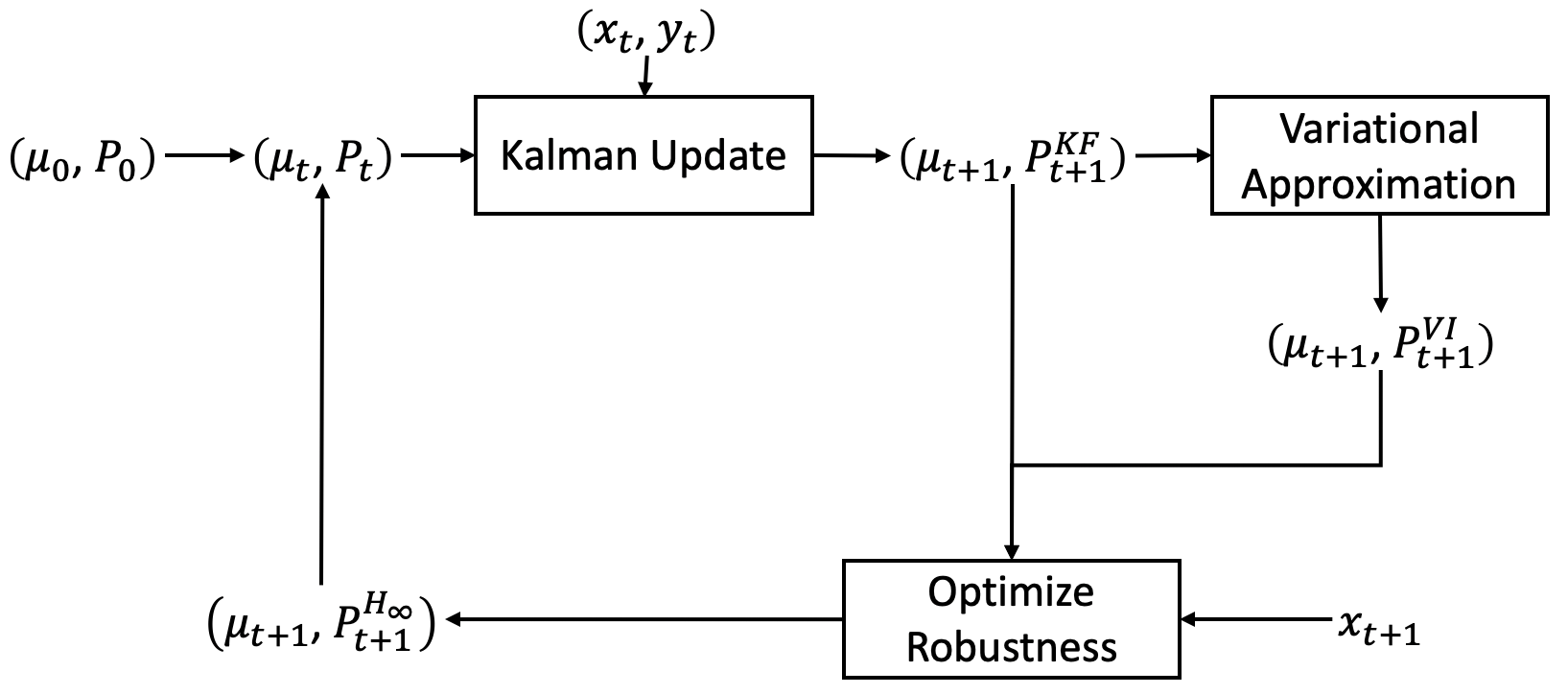}
   \caption{Illustration of a robust Variational Filter. The Filter is composed of three primary steps. First the Kalman update assimilates data using Bayesian Inference to compute the true update $(\mu_{t+1},\vo{P}^{\text{KF}}_{t+1})$. Then Variational Inference is used to find the best low memory approximation $\vo{P}^{VI}_{t+1}$ (using either EP or $\mathbb{L}^r$). Finally, comparing the true update and variational approximation a final update $\vo{P}^{H_\infty}_{t+1}$is optimized to robustly assimilate the next observation at $\vo{x}_{t+1}$}
   \label{figurelabel1}
\end{figure}

\begin{algorithm}
%\DontPrintSemicolon % Some LaTeX compilers require you to use \dontprintsemicolon    instead
%\KwIn{$\vo{P}^{H_{\infty}}_{t-1}$: \text{Diagonal posterior matrix at step (t-1)}}
%\KwOut{$\vo{P}^{H_{\infty}}_{t}$: \text{Diagonal posterior matrix at step t}}
$\vo{P}^{H_{\infty}}_{t-1}$: \text{Diagonal posterior matrix at step (t-1)}
$\vo{P}^{H_{\infty}}_{t}$: \text{Diagonal posterior matrix at step t}
\renewcommand{\labelenumi}{(\Roman{enumi})}
% \begin{enumerate}[noitemsep,nolistsep]
% \item
 Calculate: 
 \begin{align}
  \vo{K}_t &= (\vo{A}_{t-1}\vo{P}^{H_{\infty}}_{t-1}\vo{A}_{t-1}^T+\vo{Q})\x_t\nonumber\\&\times\Big[\x_t^T(\vo{A}_{t-1}\vo{P}^{H_{\infty}}_{t-1}\vo{A}_{t-1}^T+\vo{Q})\x_t+\vo{R}\Big]^{-1}\nonumber\\
  \vo{P}_t^{\text{KF}} &= (\vo{I}-\vo{K}_t\vo{x}_t^T)(\vo{A}_{t-1}\vo{P}^{H_{\infty}}_{t-1}\vo{A}_{t-1}^T+\vo{Q})
\end{align}
\renewcommand{\labelenumi}{(\Roman{enumi})}

% \item 
Optimization 1:
\begin{align}
 \vo{d}^{\mathbb{L}^r} = \argmin_{\vo{d}_t} \int p(\vTh \mid \vo{y}_t)\Big|\text{log}\frac{p(\vTh\mid \vo{y})}{q(\theta \mid \vo{d}_t)}\Big|^r d\vTh \rightarrow \vo{P}^{\mathbb{L}^r}_t(\vo{d})
\end{align}
\renewcommand{\labelenumi}{(\Roman{enumi})}

Propagate $\vo{P}^{\mathbb{L}^r}_t(\vo{d})$: $\tilde{\vo{P}}^{\mathbb{L}^r}_{t+1} = (\vo{A}_{t+1}\vo{P}^{\mathbb{L}^r}_t\vo{A}_{t+1}^T+\vo{Q})$
% \begin{align}
%     \tilde{\vo{P}}^{\mathbb{L}^r}_{t+1} = (\vo{A}_{t+1}\vo{P}^{\mathbb{L}^r}_t\vo{A}_{t+1}^T+\vo{Q})
% \end{align}
% \item

Compute the Kalman update for the next step using $\vo{P}_t^{\text{KF}}$:
\begin{align}
\vo{K}^{\text{KF}}_{t+1} &= (\vo{A}_{t}\vo{P}^{\text{KF}}_{t}\vo{A}_{t}^T+\vo{Q})\x_{t+1}\nonumber\\&\times\Big[\x_{t+1}^T(\vo{A}_{t}\vo{P}^{\text{KF}}_{t}\vo{A}_{t}^T+\vo{Q})\x_{t+1}+\vo{R}\Big]^{-1}
\end{align}
% \item 
\renewcommand{\labelenumi}{(\Roman{enumi})}

Compute the expression for $\vo{K}^{\text{H}_{\infty}}_{t+1}$:
\begin{align}
    \vo{K}^{\text{H}_{\infty}}_{t+1} =& \tilde{\vo{P}}^{\mathbb{L}^r}_{t+1}[\vo{I}-\gamma \tilde{\vo{P}}^{\mathbb{L}^r}_{t+1} + \vo{x}_t\vo{R}^{-1}\vo{x}_t^T\tilde{\vo{P}}^{\mathbb{L}^r}_{t+1}]^{-1}\nonumber\\
    &\times\vo{x}_t\vo{R}^{-1}
\end{align}

Optimization 2:
\begin{align}
  \gamma^*_t &= \argmin_{\gamma} ||\vo{K}^{\text{H}_{\infty}}_{t+1}-\vo{K}^{\text{KF}}_{t+1}||_2 \nonumber\\
  \vo{P}^{H_{\infty}}_t  &= \vo{P}^{\mathbb{L}^r}_t(\vo{d})[\vo{I}-\gamma_t^* \vo{P}^{\mathbb{L}^r}_t(\vo{d}) + \vo{x}_t\vo{R}^{-1}\vo{x}_t^T\vo{P}^{\mathbb{L}^r}_t(\vo{d})]^{-1}
\end{align}

% \renewcommand{\labelenumi}{(\Roman{enumi})}

% Express the Kalman update for the next step using $\vo{P}^{\mathbb{L}^r}_t(\vo{d})$ and equation.~\eqref{eq:Khinf}:
% \begin{align}
% \vo{K}^{\text{H}_{\infty}}_{t} =& \vo{P}^{\mathbb{L}^r}_t(\vo{d})[\vo{I}-\gamma \vo{S}_t\vo{P}^{\mathbb{L}^r}_t(\vo{d}) + \vo{x}_t\vo{R}^{-1}\vo{x}_t^T\vo{P}^{\mathbb{L}^r}_t(\vo{d})]^{-1}\nonumber\\&\times\vo{x}_t\vo{R}^{-1}
% \end{align}
% % \item 
% \renewcommand{\labelenumi}{(\Roman{enumi})}

% Optimization 2:
% \begin{align}
%   &\gamma^*_t = \argmin_{\gamma} ||\vo{K}^{\text{H}_{\infty}}_{t}-\vo{K}^{\text{KF}}_{t}||_2 \nonumber\\
%   &\vo{P}^{H_{\infty}}_t  = \vo{P}^{\mathbb{L}^r}_t[\vo{I}-\gamma^*_t\vo{P}^{\mathbb{L}^r}_t]^{-1}
% \end{align}

% \renewcommand{\labelenumi}{(\Roman{enumi})}

% \end{enumerate}
\caption{Algorithm of Augmented H$_{\infty}$ Filter Update}
\label{algo:b}
\end{algorithm}

When the Optimization step 1 is replaced by the EP objective, we recover the VI-H$_{\infty}$ filtering.
The fig.~\ref{figurelabel1} illustrates how this augmented filtering is performed. It is important to note that all these steps can be done in a computationally efficient manner, assuming $\vo{A}_t$ is sparse, without having to manipulate dense matrices. When we assume that $\vo{A}_t$ is the identity, as is the case for parameter estimation, computing is highly efficient, even when working with $\vo{P}_t^{\text{KF}}$ because it is a diagonal plus low-rank matrix.

\section{Results}
\begin{figure}
   \centering
   \includegraphics[width=0.5\textwidth]{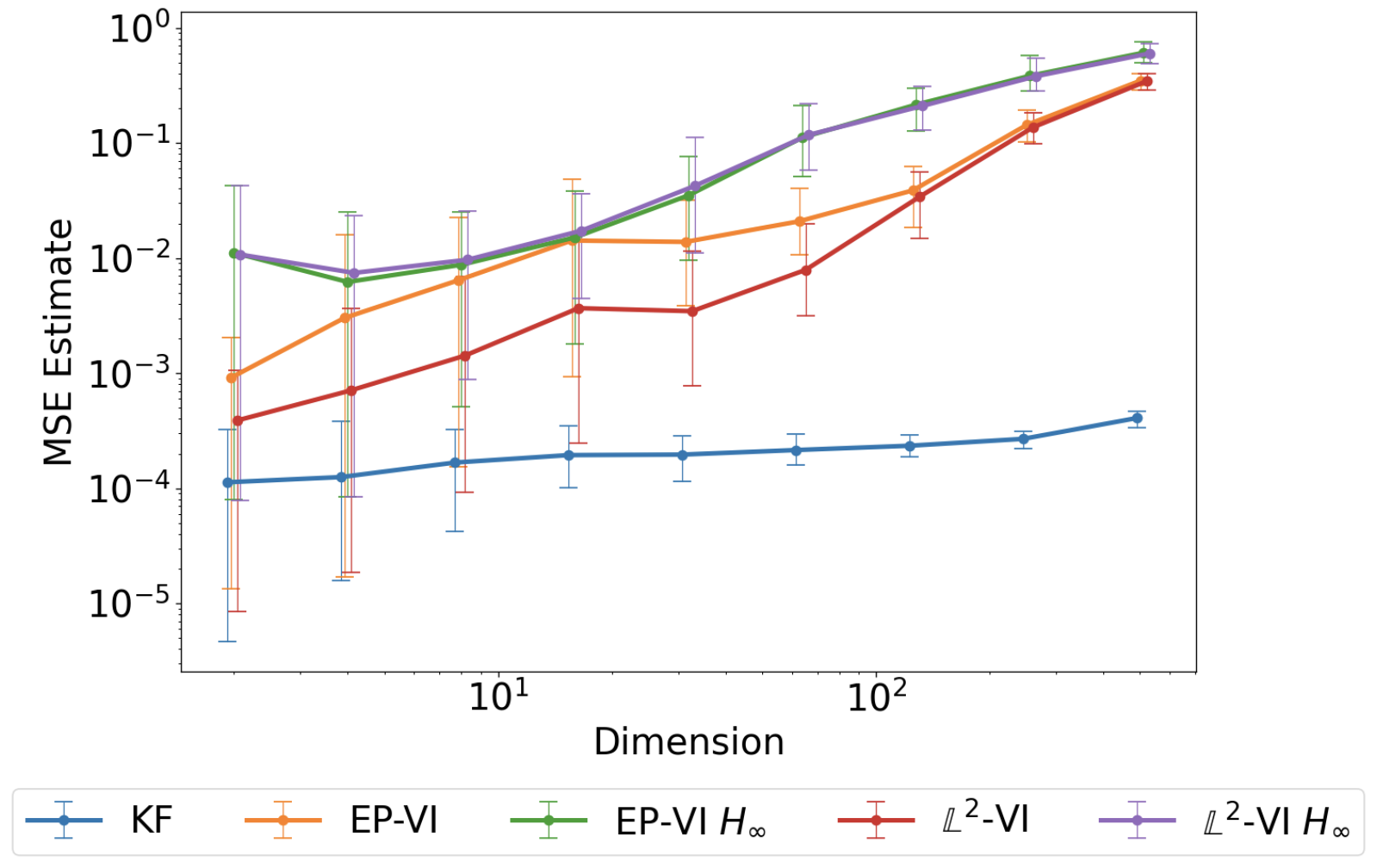}
   \caption{Mean Square Error for the five different filtering algorithms with varying problem dimensions after 1000 sequentially observed data points. The error bars correspond to 93\% confidence intervals. Two main observations are that 1) the $\mathbb{L}^2$ formulation of the VI problem outperforms standard EP VI and 2) the two $H_\infty$ filters under-perform, but as we will see are not overconfident so are more trustworthy. There is a small perturbation on the different filter's dimension coordinates to improve readability. }
   \label{figurelabel2}
\end{figure}
\begin{figure}
   \centering
   \includegraphics[width=0.5\textwidth]{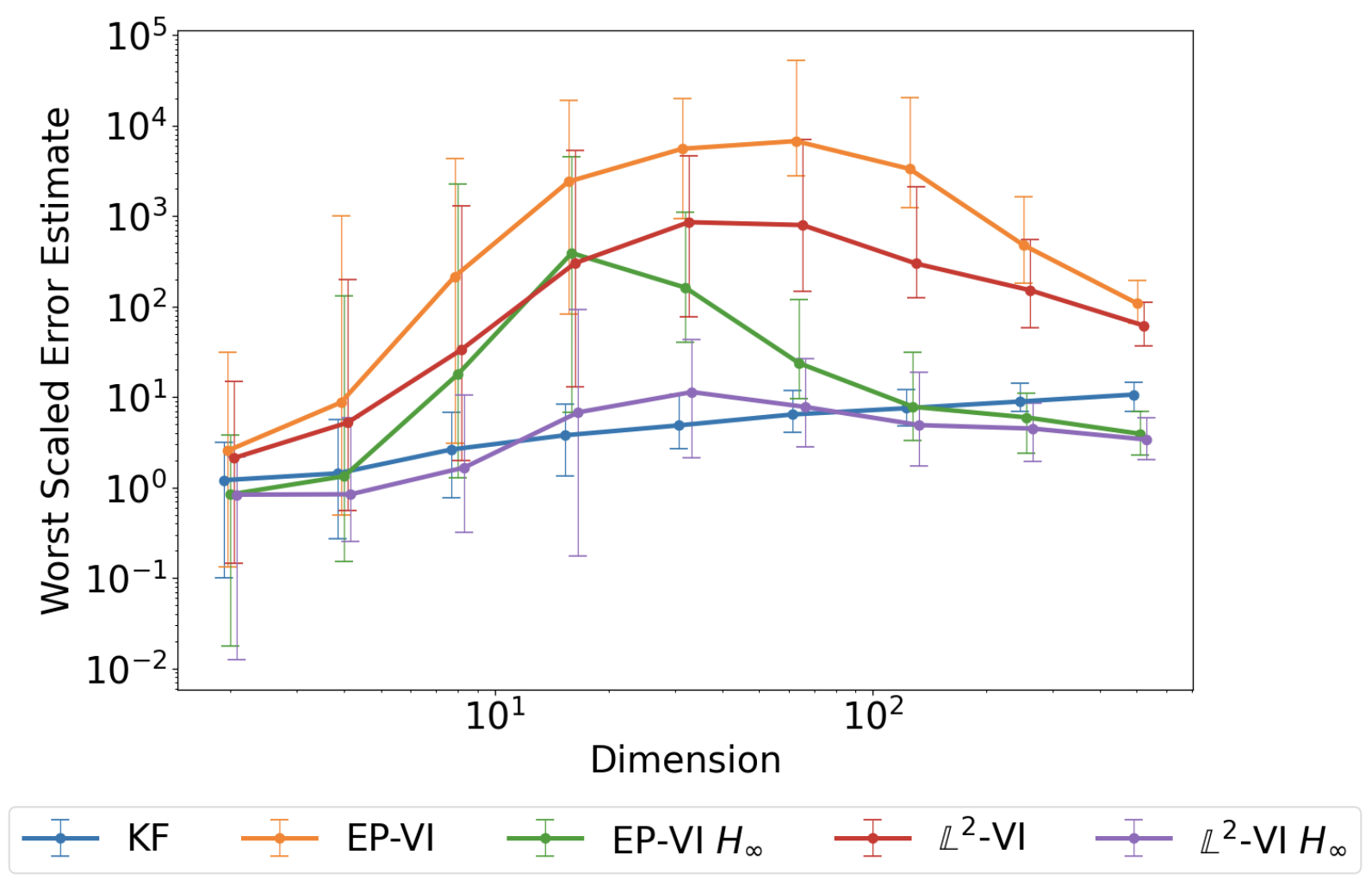}
   \caption{Worst Case Scaled Error (e.g. the largest absolute error divided by estimated standard deviation) for the five  filters. The error bars are 93\% confidence intervals. We see that 1) the variational filters without added robustness can be significantly overconfident 2) the two $H_\infty$ filters, particularly the $\mathbb{L}^2$ $H_\infty$ better capture the worst case estimate uncertainty. These methods trade robustness and lower bias for slower convergence. There is a small perturbation on the different filter's dimension coordinates to improve readability.}
   \label{figurelabel3}
\end{figure}
We consider comparing 5 different methods for parameter estimation. They are: the Kalman filter, the variational inference filter using the EP objective, the augmented VI-H$_{\infty}$ filter, the $\mathbb{L}^2$ filter, and the augmented $\mathbb{L}^2$-H$_{\infty}$ filter. Without loss of generality, we assume no system dynamics and no process noise, focusing only on the sequential update in the parameter estimate $\hat{\vTh_t}$, conditioned on the data $\vo{y}_t$. We consider the case when dimension of y is set to be 1, emulating our systems of interest, with much lower observation dimension compared with the parameter dimension. Our observation vector is $x_t \sim \mathcal{N} \left (\vXBar, 0.5 \vo{I}  \right)$. For each problem a $\vXBar$ was selected from $\mathcal{N} \left (0, \vo{I}  \right)$. The observations had additive measurement noise with variance $0.1$. The prior on $\theta_0$ was $\mathcal{N} \left (0, \vo{I} \right)$. The filter observes $1000$ time steps. This choice of problem is more difficult than if we had included process noise or a larger dimensional observation. Process noise allows estimation errors to become less important over time while larger dimensional observation spaces allow for the update to include more information without having to try to sequentially assimilate that same information over multiple observation steps as in our case.

In fig.~\ref{figurelabel2}, we compare the MSE estimate of the 5 filters with varying dimension of the problem size from 2 to 512 in powers of 2. The mean and range of MSE was estimated from 32 randomized problems for each dimension. The figure also shows the uncertainty in those estimates. The Kalman filter response with no imposed structure on the $\vo{P}_t$ matrix, performs the best of all, as expected since it is optimal. The availability of cross-correlation terms allows for the lowest MSE estimates. It is interesting to note that the information pseudo-metric based filter outperforms the traditional VI filter, with both having diagonal $\vo{P}_t$. The other two proposed augmented H$_{\infty}$ filters both under-perform with respect to the MSE estimates. However, it should be pointed out that the gap between the MSEs of the VI and $H_{\infty}$ filters seems to decrease respect to the increase in dimension of the problem. We further see that the MSE estimates fluctuates the least among them too. We assert that this will lead to better convergence characteristics of the posterior with increase in the problem size at the cost of time of convergence, which is suitable for many parameter estimation problems. The diagonal approximation of the $\vo{P}_t$ certainly introduces an error with respect the true posterior. In fig.~\ref{figurelabel3} we show the worst case scaled error performance of the 5 filters. The scaling with respect to the estimated standard deviation for each parameter and the worst performance was selected. When this metric is larger than the Kalman filter, then that filtering algorithm is unjustifiably confident in its estimate. Conversely when it is smaller then it is under confident. Among the filters with $\vo{P}_t$ diagonal, we see that $\mathbb{L}^2$-H$_{\infty}$ outperforms the rest. Without the H$_{\infty}$ correction the VI filter and $\mathbb{L}^2$ are expected to performing worse than the ones with H$_{\infty}$ correction, which is shown in the figure. This highlights that the $\mathbb{L}^2$-H$_{\infty}$ is very robust meaning that it's predictions are trustworthy.

\section{Conclusion}
The standard KF becomes exceedingly memory intensive as the dimension of the underlying state space increases. Variants of the KF have been proposed to reduce the dimension of the system, thus making implementation in high dimensions possible. The reduced rank KF project the state/parameter vector of the model onto a lower dimensional subspace. The success of this approach depend on meticulously choosing the reduction operator which may not be possible when the effective problem dimension is large.

In this paper, we propose a reduced memory filter based upon variational inference with a information psuedometric and H$_{\infty}$ filter in order to resolve the storage and computational issue related to the error covariance matrix while retaining robustness. With VI, we approximate the full rank $n\times n$ covarianace matrix, with a matrix characterized by $n$ elements along its diagonal. The storage and computational complexity of the operations is now $O\left(n\right)$. In order to test our proposed filters we consider test cases with increasing problem size to analyze its impact of our method for large dimensional problems. Our filters, specifically the $\mathbb{L}^2$-H$_{\infty}$ exhibits slower learning, but enables information update along all the directions of the state space, keeping the worst case performance better than other filters. In our future work we would like to analyze properties of our proposed filter and experiment with real-world high dimensional problems, particularly those emerging in Machine Learning.
\printbibliography
\end{document}